\documentclass{article}
\pdfoutput=1
\usepackage[utf8]{inputenc} 
\usepackage[T1]{fontenc}    
\usepackage[numbers]{natbib}
\usepackage{mathrsfs} 
\usepackage{amsmath,amsbsy,amsfonts,amssymb,amsthm,dsfont,units}
\usepackage{mathtools}
\usepackage{color,cases}
\usepackage{tikz}
\usetikzlibrary{calc,shapes}
\usepackage{caption}
\usepackage{subcaption}
\usepackage[utf8]{inputenc}
\usepackage[english]{babel}
\usepackage{multirow}
\usepackage{bbm}
\usepackage{fullpage}
\usepackage{graphicx}

\newcommand{\ALGNAME}{GGT}
\newcommand{\SGDf}{ \widetilde{\nabla} f }

\makeatletter
\newcommand\footnoteref[1]{\protected@xdef\@thefnmark{\ref{#1}}\@footnotemark}
\makeatother

\usepackage{algorithm}
\usepackage{algpseudocode}
\algnewcommand\algorithmicinput{\textbf{Input: }}
\algnewcommand\INPUT{\State\algorithmicinput}
\algnewcommand\algorithmicinitialize{\textbf{Initialize: }}
\algnewcommand\INIT{\State\algorithmicinitialize}
\algnewcommand\algorithmicrun{\textbf{Run: }}
\algnewcommand\RUN{\State\algorithmicrun}
\algnewcommand\algorithmicupdate{\textbf{Update: }}
\algnewcommand\UPDATE{\State\algorithmicupdate}
\algnewcommand\algorithmicset{\textbf{Set: }}
\algnewcommand\SET{\State\algorithmicset}
\algnewcommand\algorithmicquery{\textbf{Query: }}
\algnewcommand\QUERY{\State\algorithmicquery}
\algnewcommand\algorithmicoutput{\textbf{Output: }}
\algnewcommand\OUTPUT{\State\algorithmicoutput}

\def\norm#1{\mathopen\| #1 \mathclose\|}

\newcommand{\ignore}[1]{}

\def\reals{{\mathbb R}}

\def\bold0{\mathbf{0}}

\def\bI{\mathbf{I}}

\def\bV{\mathbf{V}}

\def\bI{\mathbf{I}}

\def\bV{\mathbf{V}}

\def\eps{\varepsilon}
\def\epsilon{\varepsilon}

\newcommand{\defeq}{\stackrel{\text{def}}{=}}

\newcommand{\braces}[1]{\left\{#1\right\}}
\newcommand{\pa}[1]{\left(#1\right)}

\newcommand{\bra}[1]{\left[#1\right]}

\DeclareMathOperator{\argmin}{argmin}

\newtheorem{theorem}{Theorem}[section]

\newtheorem{definition}[theorem]{Definition}
\newtheorem{lemma}[theorem]{Lemma}

\theoremstyle{definition}

\DeclareMathOperator*\E{\mathbb{E}}
\newcommand\R{\mathbb{R}}

\newcommand{\G}{\mathbf{G}}
\newcommand{\bU}{\mathbf{U}}
\newcommand{\bSig}{\mathbf{\Sigma}}

\title{Efficient Full-Matrix Adaptive Regularization}

\author{
  Naman Agarwal$^1$ \qquad Brian Bullins$^{3}$\footnote{Work performed at Princeton University and Google AI Princeton.} \qquad Xinyi Chen$^{1,2}$ \\
  Elad Hazan$^{1\,2}$ \qquad Karan Singh$^{2*}$ \qquad Cyril Zhang$^{4*}$ \qquad Yi Zhang$^{2*}$ \\
  \vspace{-2mm} \\
  $^1$ {\small Google AI Princeton} \\
  $^2$ {\small Department of Computer Science, Princeton University} \\
  $^3$ {\small Toyota Technological Institute at Chicago} \\
  $^4$ {\small Microsoft Research} \\
  \vspace{-2mm} \\
  {\small \texttt{namanagarwal@google.com}, \texttt{bbullins@ttic.edu},} \\
  {\small \texttt{\{xinyic, ehazan, karans, y.zhang\}@cs.princeton.edu}, } \\
  {\small \texttt{cyrilzhang@microsoft.com} }
}

\begin{document}

\maketitle

\begin{abstract}
Adaptive regularization methods pre-multiply a descent direction by a
preconditioning matrix. Due to the large number of parameters of machine
learning problems, full-matrix preconditioning methods are prohibitively
expensive. We show how to modify full-matrix adaptive regularization in order
to make it practical and effective. We also provide a novel theoretical analysis
for adaptive regularization in {\em non-convex} optimization settings. The
core of our algorithm, termed GGT, consists of the efficient computation of
the inverse square root of a low-rank matrix. Our preliminary experiments show
improved iteration-wise convergence rates across synthetic tasks and
standard deep learning benchmarks, and that the more carefully-preconditioned steps
sometimes lead to a better solution.
\end{abstract}


\section{Introduction}

Stochastic gradient descent is the workhorse behind the recent deep learning revolution. This simple and age-old algorithm has been supplemented with a variety of enhancements to improve its practical performance, and sometimes its theoretical guarantees.

Amongst the acceleration methods there are three main categories: momentum, adaptive regularization, and variance reduction. Momentum (in its various incarnations, like heavy-ball or Nesterov acceleration) is the oldest enhancement. It has a well-developed theory, and is known to improve practical convergence in a variety of tasks, small and large. It is also easy to implement.  Variance reduction is the most recent advancement; in theory and practice, it is mostly applicable to convex optimization, and is thus less influential in deep learning. 

This brings us to adaptive regularization: the most sophisticated, hard to implement, and widely-debated acceleration method. While state-of-the-art optimizers such as Adam and AdaGrad \citep{kingma2014adam, duchi2011adaptive} do use adaptive regularization, they exclusively restrict the preconditioner to a diagonal matrix; as such, they often marketed as per-coordinate ``adaptive learning-rate'' methods. Despite solid theoretical guarantees, the practical value of diagonal adaptive regularization as compared to ``vanilla'' SGD has been the subject of much debate \citep{wilson2017marginal}. 
However, the efficacy of full-matrix adaptive regularization has been relatively unexplored. This is due to the prohibitive computational cost associated with full-matrix operations: full AdaGrad requires taking the inverse square root of a large matrix. 

In this paper, we present \ALGNAME{}, a practical solution to the computational problems plaguing full-matrix adaptive regularization, making this technique scalable for modern deep models. At the heart of our method is a simple, GPU-friendly way to apply the inverse square root of the low-rank second-moment matrix of recent gradients; see Figure~\ref{fig:inverse-formula}. \ALGNAME{}'s running time is comparable to state-of-the-art optimizers.

We proceed to show that full-matrix preconditioning allows for much better exploitation of anisotropic curvature in loss landscapes. First, we show synthetic experiments which demonstate clear benefits of \ALGNAME{} over baselines, especially when the problem is ill-conditioned. Then, we implement \ALGNAME{} at scale, and show that the benefits translate to faster training on standard deep learning benchmarks.
Our improvement is most salient in complicated landscapes like RNN training.

Our algorithm comes with theoretical guarantees. We give the first proof of convergence to first-order critical points for an algorithm with adaptive regularization in a stochastic non-convex setting, featuring a rate which is dependent on an \emph{adaptive ratio}. We show examples where our bound is stronger than that for SGD, providing some theoretical basis for our empirical findings.

\subsection{Related work}
Since the introduction of AdaGrad \citep{duchi2011adaptive}, diagonal adaptive regularization has been a mainstay
in the machine learning practitioner's toolbox.
A quick perusal of the literature shows that these methods have continued to thrive in the deep learning era,
and appear in all major frameworks \citep{abadi2016tensorflow,paszke2017automatic,chen2015mxnet}.
By citation count (or GitHub search hits), Adam \citep{kingma2014adam} is by far
the most popular adaptive optimizer for training a variety of modern deep models.
For this reason, this paper's exposition is targeted towards a full-matrix drop-in replacement for Adam;
however, our techniques extend straightforwardly to a plethora of variants, like
RMSprop \citep{tieleman2012lecture}, Adadelta \citep{zeiler2012adadelta},
Nadam \citep{dozat2016incorporating}, etc.

Full-matrix adaptive regularization has existed alongside the more commonly used diagonal-matrix manifestation
since their common inception in \citep{duchi2011adaptive}; however, a major obstacle to the scalability of these
methods is the need for the storage and inversion of square matrices in the model dimension.
This becomes prohibitively expensive in dimension greater than $10^4$, while state-of-the-art models regularly exceed $10^7$ parameters.

Matrix sketching has been employed to approximate the AdaGrad preconditioner \citep{krummenacher2016scalable,mehta2016compadagrad};
however, the sketched estimate for the matrix inverse can be sensitive to noise.
Furthermore, there is a sizeable performance gap which renders these methods unsuitable for large-scale implementation.
For example, in the former reference, the authors report a 5-10$\times$ overhead over AdaGrad, even with $<\!10^5$ model parameters;
we could not find a usable GPU implementation for their requisite rank-1 QR update. Rather than sketching the preconditioner, our computations are exact, and efficient in practice.

\citep{gupta2018shampoo} propose a way to do AdaGrad with Kronecker products of full-matrix preconditioners, a more limited setting which requires knowledge of the model's structure.
Finally, as we argue in Section~\ref{subsection:experiments-synthetic},
there is intrinsic value of ``forgetting'' past curvature using an exponential window.
With this, a low-rank preconditioning matrix naturally arises,
allowing us to bypass the computational need for matrix sketching in the model dimension
or architecture-dependent restriction of the preconditioner.

Our algorithm bears a superficial resemblance to L-BFGS \citep{liu1989limited}:
both compute a preconditioner using a sliding window of gradient history.
However, L-BFGS uses differences of gradients to estimate the Hessian, while we use the gradients to keep an exact (exponentially decayed) gradient Gram matrix. In the former, estimating the Hessian with \emph{stochastic} gradients is very unstable. In a similar vein, stochastic second-order methods \citep[e.g.][]{erdogdu2015convergence,agarwal2017second,luo2016efficient,hazan2007logarithmic,agarwal2017finding,carmon2017convex} have seen limited adoption in deep learning. This is partly due to prohibitive overhead costs (despite asymptotic efficiency); however, more fundamentally, they exhibit very different training dynamics than the typical family of optimizers in deep learning.

Natural gradient descent \citep{amari1998natural} and K-FAC \citep[e.g.][]{martens2015optimizing,martens2018kronecker} precondition the gradient by the inverse (rather than inverse square root) of the gradient Gram matrix, or an efficiently invertible approximation thereof. Like in the above discussion, the training dynamics arising from these methods diverge from those of standard adaptive optimizers, as well as the underlying theoretical motivations.

Several recent works \citep{li2018convergence,zou2018convergence,ward2018adagrad,chen2018convergence,chen2018universal,zhou2018convergence} have studied the convergence of adaptive methods for non-convex optimization, matching the asymptotic iteration complexity of SGD. \cite{staib2019escaping} show that adaptive methods escape saddle points. Apart from our algorithmic contribution, our work is (to our knowledge) the first attempt to characterize the \emph{advantage} of adaptivity in terms of the dimension and geometry of the optimization problem.

\begin{figure}
  \centering
  \includegraphics[width=0.7\linewidth]{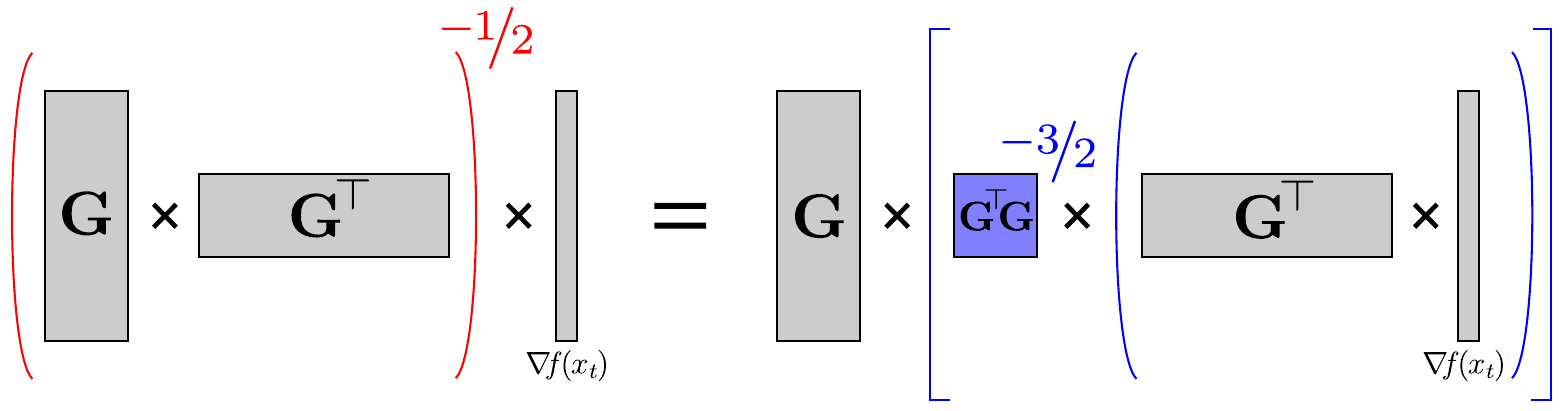}
  \caption{ Sketch of how \ALGNAME{} performs fast full-matrix preconditioning.
  Note that the inverse matrices are understood here to be Moore-Penrose pseudoinverses;
  see Section~\ref{subsection:alg-implementation} for a full treatment. }
  \label{fig:inverse-formula}
\end{figure}


\section{The \ALGNAME{} algorithm}

Our main algorithmic contribution is \ALGNAME{}, an efficient first-order algorithm for full-matrix adaptive preconditioning.
In brief, \ALGNAME{} uses the preconditioner from full-matrix AdaGrad, with gradient history
attenuated exponentially as in Adam, and truncated to a window parameter $r$.
The name \ALGNAME{} acts as a convenient mnemonic for the gradient second-moment matrix $\G\G^\top$
maintained by full-matrix AdaGrad, even though \emph{we never compute this matrix}.

The mathematical specification of \ALGNAME{} is given in Algorithm~\ref{alg:ours},
in the usual model of stochastic optimization (see Section~\ref{sec:theory}),
with gradients $\widetilde\nabla \! f(x)$. Notice that the coordinate-wise scaling of Adam is recovered by zeroing out the off-diagonal entries of $\G\G^\top$.

\begin{algorithm}
\caption{ \ALGNAME{} adaptive optimizer }
\begin{algorithmic}[1]
\INPUT initializer $x_1$, window size $r$, learning rate schedule $\braces{\eta_t}$, $\beta_2 \leq 1, \eps > 0$.
\For{$t = 1, \ldots, T$}
  \State Receive stochastic gradient $\widetilde\nabla f(x_t)$.
  \State Let $\G_t = [g_t\ g_{t-1}\ \dots\ g_{t-r+1}]$, where $g_{t-k} := \beta_2^k \widetilde\nabla \! f(x_{t-k})$,
  or $\mathbf{0}$ if $k \geq t$.
  \State $x_{t+1} \leftarrow x_t - \eta_t \cdot [ (\G_t\G_t^\top)^{1/2} + \eps \bI ]^{-1} \, \widetilde \nabla \! f(x_t)$.
\EndFor
\end{algorithmic}
\label{alg:ours}
\end{algorithm}

GGT provides the power of full-matrix adaptive regularization at a cost not much larger than SGD. This crucially exploits the fact only a small window of historical gradients are used for preconditioning. The intuition for using a small window, as opposed to the entire history, is clear (and time-tested, by the ubiquity of Adam): the curvature of the loss surface changes, rendering previous gradient information obsolete. We expand on the benefits of forgetting gradients in section \ref{subsection:experiments-synthetic}. 

The fact that the preconditioning matrix is based on a small window of gradients implies that it has low rank. GGT exploits this fact by computing the inverse square root of the empirical covariance matrix indirectly, as outlined in Figure~\ref{fig:inverse-formula}. In effect, instead of inverting a full matrix in the dimension of parameters, using the special matrix structure GGT inverts a matrix of dimension window-size. 
The remainder of this section will discuss efficient implementation and some heuristics.

GGT has provable guarantees even for non-convex optimization: it is guaranteed to converge to a first-order critical point. Its rate of convergence is never significantly slower than that of SGD, and in some favorable geometric conditions, can be significantly faster. These theoretical bounds are made precise in Section~\ref{sec:theory}.

\subsection{Fast low-rank preconditioning}
\label{subsection:alg-implementation}

The window parameter $r$ should be roughly the number of copies of the model that fit in RAM;
in our large-scale experiments, we use $r = 200$.
A pessimistic but principled choice is $r = \Theta( 1/(1-\beta_2) )$,
which truncates on the time scale of the exponential attenuation.
Our key observation, highlighted in Figure~\ref{fig:inverse-formula},
is that the inversion of the large low-rank matrix $\G\G^\top$ can be performed by
diagonalizing the small matrix $\G^\top\G$, along with some extremely GPU-friendly matrix-vector operations.

The basic intuition is contained in Figure~\ref{fig:inverse-formula}, but it remains
to include the $\eps \bI$ term. We derive the full update here.
Let $\G \in \R^{d \times r}$, $v \in \R^d$ be arbitrary, with $r \leq d$.
Write the singular value decomposition $\G = \bU\bSig \bV^\top$,
with $\bU \in \R^{d \times d}, \bSig \in \R^{d \times r}, \bV \in \R^{r \times r}$.
Let $\bSig_d \in \R^{d \times d} := [\bSig\;0]$,
and let $\bSig_r \in \R^{r \times r}$ be its top left block.
Let $\bU =: [\bU_r\ \bU_{d-r}]$, so that the columns of
$\bU_r \in \R^{d \times r}$ are an orthonormal basis for the column space of $\G$,
and $\bU_{d-r} \in \R^{d \times (d-r)}$ its orthogonal component, noting that $\bU_r \bU_r^\top + \bU_{d-r} \bU_{d-r}^\top = \bI_{d}$.
Then, we have
\begin{align*}
\big[ (\G&\G^\top)^{1/2} + \eps \bI \big]^{-1} v = \left[ (\bU\bSig_d^2 \bU^\top)^{1/2} + \eps \bU\bU^\top \right]^{-1} v \\
&= \bU(\bSig_d + \eps \bI )^{-1}\bU^\top v \\
&= \left[ \bU_r(\bSig_r + \eps \bI_{r})^{-1}\bU_r^\top + \bU_{d-r}( \eps \bI_{d-r} )^{-1}\bU_{d-r}^\top \right] v \\
&= \bU_r(\bSig_r + \eps \bI_r)^{-1}\bU_r^\top v + \frac{1}{\eps}(\bI_d - \bU_r\bU_r^\top) v \\
&= \frac{1}{\eps}v + \bU_r \bra{ (\bSig_r + \eps \bI_r)^{-1} - \frac{ 1 }{\eps} \bI_r } \bU_r^\top v.
\end{align*}
The first term is none other than an SGD update step. The rest can be computed
by taking the eigendecomposition $\G^\top \G = \bV \bSig_r^2 \bV^\top$,
giving $\bU_r = \G \bV \sqrt{\bSig_r}^\dagger$.
We prefer this to taking the direct SVD of $\G$, which is $> \! 10$ times slower on GPU.

Using a cyclic buffer to store and update $\G_t$, the algorithm takes $O(dr^2 + r^3)$ (sequential) time per iteration,
and $O(dr)$ memory in total. Iterating over the model parameters to update $\G_t$ incurs the same
overhead cost as usual adaptive optimizers.
The $r\times d$ matrix multiplication and $r\times r$ SVD operations
benefit from decades of extensive hardware-level optimizations.

In the experiments in Section~\ref{section:experiments}, we observed a $\sim \! 1.3\times$ (CNN) and $\sim \! 2\times$ (RNN) running-time overhead compared to SGD; we note that this ratio could be even smaller in reinforcement learning (where the environment causes the time bottleneck),
or universally with a more optimized implementation.

\subsection{Tweaks for \ALGNAME{} on deep models}
\label{subsection:alg-heuristics}

Below, we list some practical suggestions for applying \ALGNAME{} to training large-scale models.

\textbf{Momentum.} In order to bring \ALGNAME{} closer to a drop-in replacement for Adam,
we can add momentum to the gradient steps: let $v_{t} \leftarrow \beta_1 v_{t-1} + \widetilde\nabla \! f(x_t)$,
and apply the preconditioner to $v_t$ to compute the update step.
We use momentum in all large-scale experiments, with the standard $\beta_1 = 0.9$.
We also get a small performance boost by using $v_t$ instead of the gradients
to update $\G_t$. On the other hand, as long as $r \ll T$,
it makes little difference to choose $\beta_2 = 1$, letting the window
(rather than exponential attenuation) forget stale gradient information.

\textbf{Interpolation with SGD.} We note the possibility of decoupling the scalars $\eps$ and $1/\eps$
which appear in the efficient update step. Appealingly, this allows the user to tune \ALGNAME{}'s
behavior to be arbitrarily close to that of SGD.

\textbf{Numerical concerns.} For greater numerical stability, it is possible to add a small
multiple of the identity matrix (we suggest $10^{-6}$) to $\G^\top \G$ before computing its eigendecomposition,
without noticeable differences in training.


\section{Experiments} 
\label{section:experiments}
In this section, we present an empirical study of \ALGNAME{}.
We begin with some simple experiments, showing that adaptive methods help in the presence of ill-conditioned optimization problems, as well as the value of limited gradient memory.
Next, we evaluate the performance of \ALGNAME{} on larger-scale deep learning tasks (and provide some additional such experiments in Appendix~\ref{sec:appendix-experiments}).
Finally, we present some interesting empirical insights on the training dynamics in deep learning models.
Our visualizations of gradient spectra suggest that adaptive optimizers are indeed correcting for changing anisotropic curvature in the loss landscape.

\subsection{Synthetic data: when do adaptivity and forgetfulness help?} 
\label{subsection:experiments-synthetic}
The original theorems on the behavior of adaptive first-order methods
are established from the perspective of online convex optimization \citep{duchi2011adaptive}.
The dynamics are less understood on realistic loss landscapes in stochastic optimization.
For this reason, we begin our experimental section with
some simple empirical comparisons between full- and diagonal-matrix adaptive optimizers and SGD.
Figure~\ref{fig:synthetic} summarizes our findings.

In each synthetic experiment, we generated an ill-conditioned landscape, and compared
SGD with adaptive optimizers, excluding the typical accompanying heuristics
(i.e. no momentum, regularization, or learning rate schedule).
We tested diagonal-matrix preconditioners with and without exponential gradient attenuation
(like Adam and AdaGrad, respectively), and their full-matrix analogues.
The experiments were robust with respect to the choice of $\eps$ (we used $10^{-4}$)
and batch size.

In the first synthetic experiment \emph{(left)}, we exhibit an instance of logistic regression in dimension 10,
with $10^3$ samples generated from an extremely anisotropic $(\sigma^2_\mathrm{max} / \sigma^2_\mathrm{min} \approx 10^4)$ Gaussian distribution,
and binary labels determined by a random hyperplane.
SGD converges the slowest, and diagonal AdaGrad consistently accelerates optimization.
Finally, full-matrix preconditioning (using cubic-time matrix inversion) converges the fastest.
In this setting, adding a window improved convergence, but not drastically; we elaborate below.

Next, we show an optimization problem \emph{(right)} which accentuates the utility of
exponentially decaying gradient memory.
We consider the problem of minimizing the logarithmic barrier function
of a randomly generated anisotropic polytope,
otherwise known as finding its \emph{analytic center}:
this replaces the logistic loss terms with $f_i(w) = -\log(w^\top x_i + c_i)$,
with $x_i$ generated the same way as above, and $c_i$ generated uniformly from $[0,1]$.
We observed the same ranking of convergence rates as in the first experiment,
but the improvement afforded by the window was much clearer.

The primary conclusion of our synthetic experiments is to demonstrate some small-scale settings in which
adaptive regularization ameliorates anisotropy in the optimization landscape.
A subtler point is that the windowed variants can help with changing curvature,
even for convex losses.
Note that the curvature of the former landscape is constant (in that its Hessian matrix
at different locations $w$ only changes by a scalar factor).
The latter setting, in contrast, features a changing curvature (its Hessians
do not commute in general), necessitating ``forgetfulness'' in adaptive curvature estimation.

In Section~\ref{subsection:experiments-spectra}, we will return to these proof-of-concept
optimization instances, connecting them to an empirical study of curvature in more realistic landscapes.

\begin{figure}
  \centering
  \includegraphics[width=0.95\linewidth]{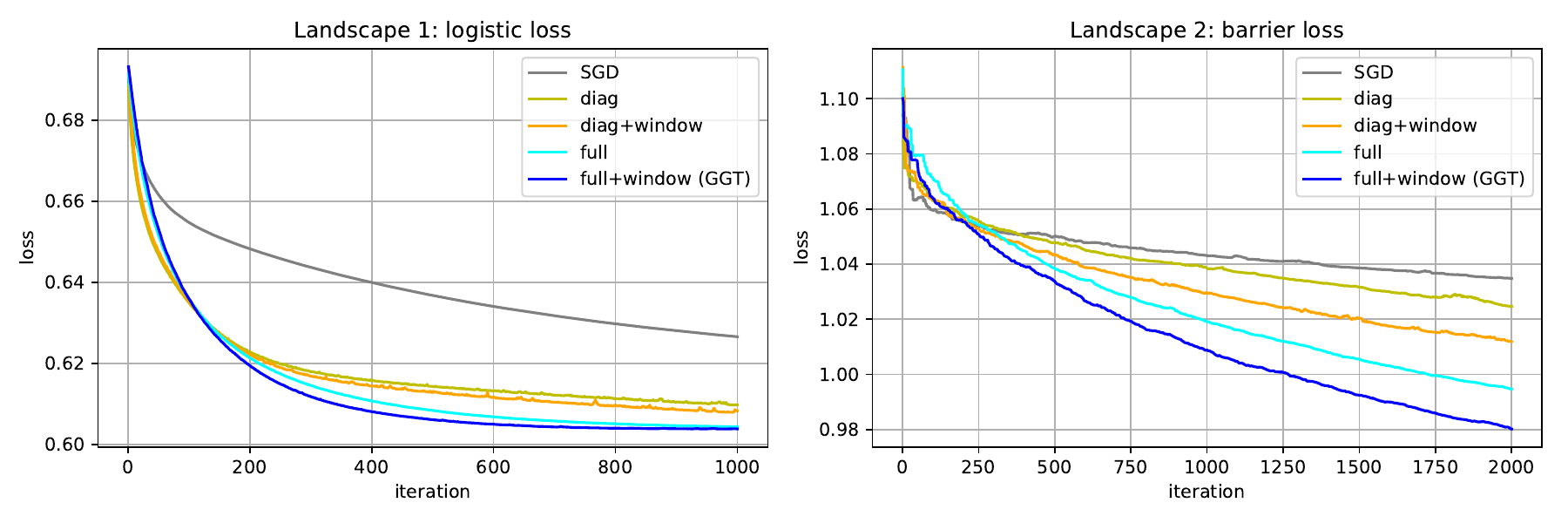}
  \caption{ Synthetic experiments on convex loss functions, demonstrating the value of adaptive regularization and attenuation of gradient history.
  \emph{Left:} An ill-conditioned instance of logistic regression. Adaptive regularization finds a good preconditioner, accelerating optimization.
  \emph{Right:} Minimizing a barrier function, an example where the curvature changes with position. Optimization is further accelerated by \emph{forgetting} outdated gradient information. }
  \label{fig:synthetic}
\end{figure}

\subsection{\ALGNAME{} on deep convolutional models} 
\label{subsection:experiments-vision}
We investigated the training dynamics of \ALGNAME{} on a typical deep architecture for computer vision.
For this, we used a 26-layer 3-branch residual network with Shake-Shake regularization \citep{gastaldi2017shake}.
Aside from its ability to reach state-of-the-art classification accuracy,
this architecture also features a relatively low parameter count ($\sim \! 3$M),
enabling the use of a large window parameter ($r=200$).

In each experiment, we kept the cosine learning rate annealing schedule
used in the paper, originally from \citep{loshchilov2016sgdr};
performance degraded consistently and significantly with a fixed learning rate.
For both Adam and \ALGNAME{}, we chose the commonly used parameters
$\beta_1 = 0.9, \beta_2 = 0.999, \eps = 10^{-8}$; for SGD, we used momentum with parameter $0.9$.
With correctly tuned RMSprop and Adadelta, with the same window parameters,
training curves were virtually identical to those for Adam.
We used a batch size of 128, and the standard data augmentation techniques of 4-pixel padding + random cropping and horizontal flipping.

Our results are shown in Figure~\ref{fig:cnn-rnn} \emph{(top)}.
In terms of training loss, \ALGNAME{} consistently dominated existing optimizers.
We corroborate a number of observations from previous empirical studies of the generalization of optimizers.
Most prominently, we found that SGD generalized slightly better than
all others \citep{wilson2017marginal,keskar2017improving} towards the end of training, including ours.
The gap $(<0.2\%)$ is less dramatic than that seen in \citep{wilson2017marginal} for two reasons:
we only show curves with a tuned and annealed learning rate;
also, we use an architecture with powerful explicit regularization techniques
which have gained attention since their publication.
Our preliminary observation is that \ALGNAME{} shrinks this gap slightly (corroborated by another experiment in Appendix~\ref{sec:appendix-experiments}),
and expect that there is vastly more empirical work to be done
concerning architectures synergistically tuned to existing optimizers.

We also verify the long-held empirical observation that the learning rate decay of AdaGrad
is too aggressive (e.g. in \citep{zeiler2012adadelta}), resulting in convergence to a poor solution.
Finally, in agreement with \citep{wilson2017marginal}, we find that using a sufficiently low
learning rate for any optimizer can result in a better training loss curve,
but not without significantly degrading generalization ($>3\%$ worse).


\subsection{\ALGNAME{} on recurrent models} 
\label{subsection:experiments-language}

\begin{figure}
  \centering
  \includegraphics[width=\linewidth]{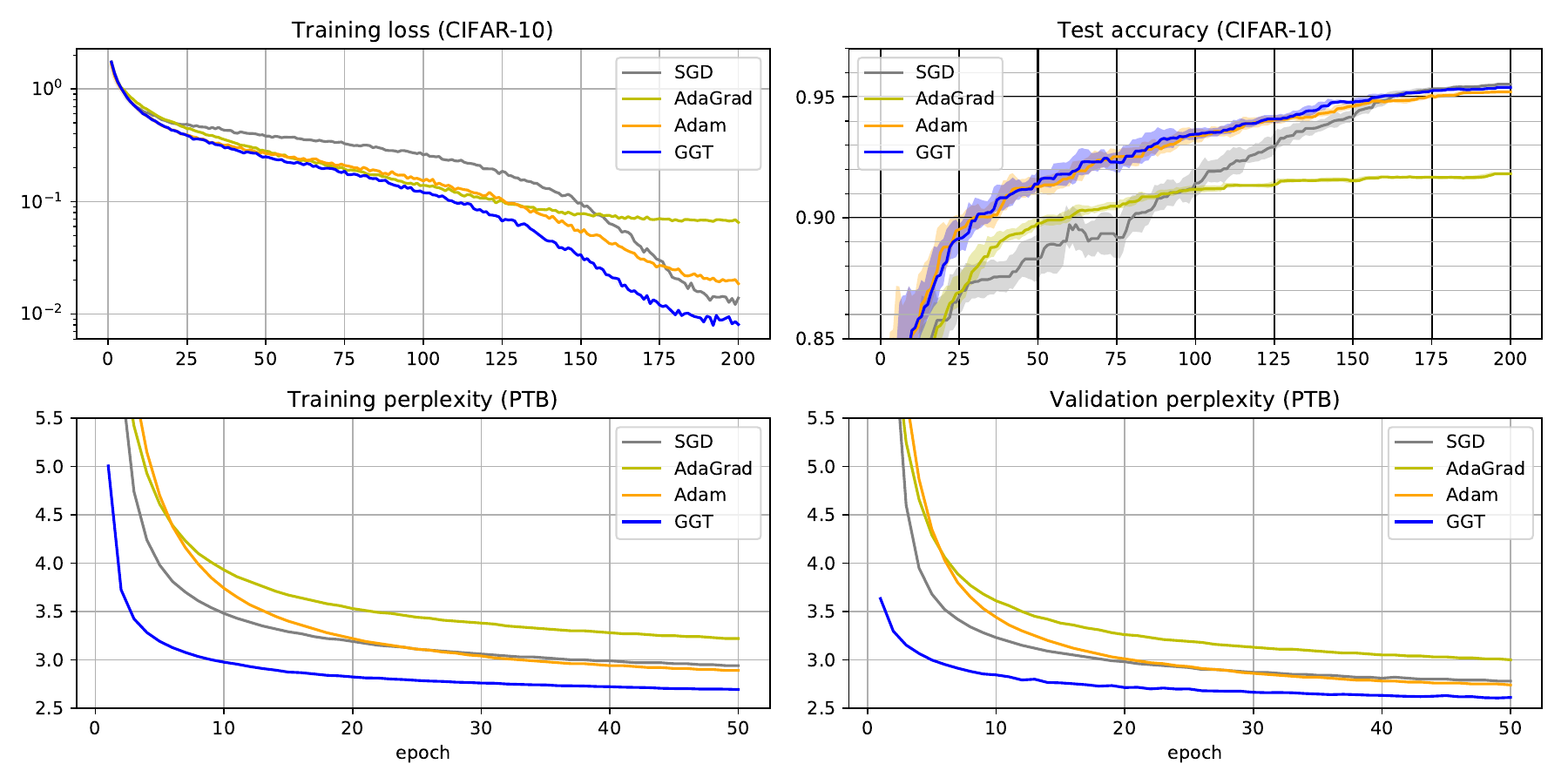}
  \caption{Results of CNN and RNN experiments. \ALGNAME{} dominates in training loss across both tasks,
  and generalizes better on the RNN task.
  \emph{Top:} CIFAR-10 classification with a 3-branch ResNet.
  \emph{Bottom:} PTB character-level language modeling with a 3-layer LSTM.
  }
  \label{fig:cnn-rnn}
\end{figure}

Next, we move to recurrent architectures for language modeling.
We train a 3-layer LSTM \citep{hochreiter1997long} with $\sim\! 5$M parameters for
character-level modeling of the Penn Treebank dataset \citep{marcus1994penn}.
This is the setting in which we observe the most striking improvement over baselines.
The particularities of this optimization task, and why it might be especially amenable
to full-matrix regularization, remain a fruitful research direction \citep{pascanu2013difficulty}.
Figure~\ref{fig:cnn-rnn} \emph{(bottom)} shows training and validation perplexities for the first $50$ epochs;
no optimizer makes significant progress afterwards.

The state of the art for character-level language modeling is less thoroughly documented
than its word-level counterpart,
though we note that our end-to-end result (validation perplexity $2.42$ after $500$ epochs) is competitive with those reported for recurrent models, like by \cite{krueger2016zoneout}. In contrast, Adam, AdaGrad, and SGD reach $2.51$, $2.65$, and $2.76$, respectively. Note that Adam is the \emph{de facto} standard optimizer for language modeling \citep{melis2017state}.
Even with iterations taking twice the time, we outperform all baselines in wall-clock time throughout training.

We also tried using \ALGNAME{} as a drop-in replacement for Adam
in the state-of-the-art word-level language modeling code
accompanying \citep{merity2017regularizing,merity2018analysis}.
Although we were competitive with Adam, we only observed an improvement
in the first $\sim \! 20$ epochs.
We hypothesize that the advantage of full-matrix regularization in this setting is more marginal,
as the gradients in the embedding layers are naturally sparse in the vocabulary (``one-hot'') basis.
On a similar note, we found that Adam outperformed GGT on attention-based architectures for NLP; refer to Appendix~\ref{sec:appendix-experiments} for an experiment and discussion.

\subsubsection{Wall clock time comparisons}

For those interested in end-to-end performance in terms of model training speed, we provide in Figure~\ref{fig:cnn-rnn-wallclock} an alternate visualization for the large-scale experiments, replacing the epoch count with total cumulative training time on the horizontal axis. On the LSTM task, GGT outperforms the baselines throughout training (and converges upon a better solution), even with the additional overhead running time.

The same unconditional improvement was not observed in the vision task, for training convergence nor generalization. We believe this is due to the interactions between modern convolutional architectures and the epoch-dependent learning rate schedule, which we have not attempted to re-tune. Indeed, recent state-of-the-art work in rapid training of convolutional neural nets is centered on a selection of learning rate and momentum schedules, rather than step directions.

\begin{figure}
  \centering
  \includegraphics[width=\linewidth]{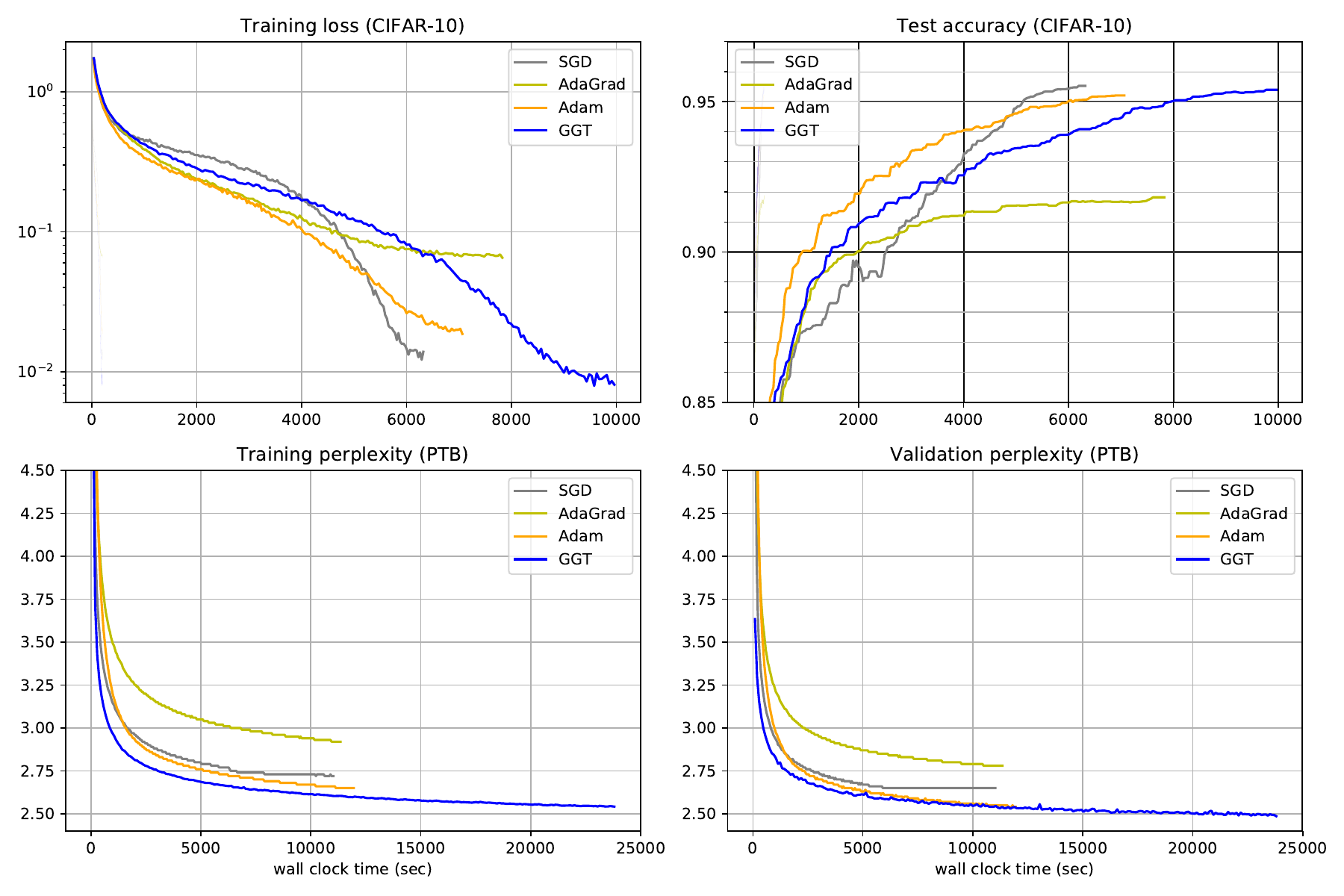}
  \caption{Alternate presentation of Figure~\ref{fig:cnn-rnn}, with wall clock time on the horizontal axis.
  \emph{Top:} CIFAR-10 classification with a 3-branch ResNet.
  \emph{Bottom:} PTB character-level language modeling with a 3-layer LSTM.
  }
  \label{fig:cnn-rnn-wallclock}
\end{figure}

\subsection{Empirical insights on the spectral decay} 
\label{subsection:experiments-spectra}

\begin{figure}
  \centering
  \includegraphics[width=0.7\linewidth]{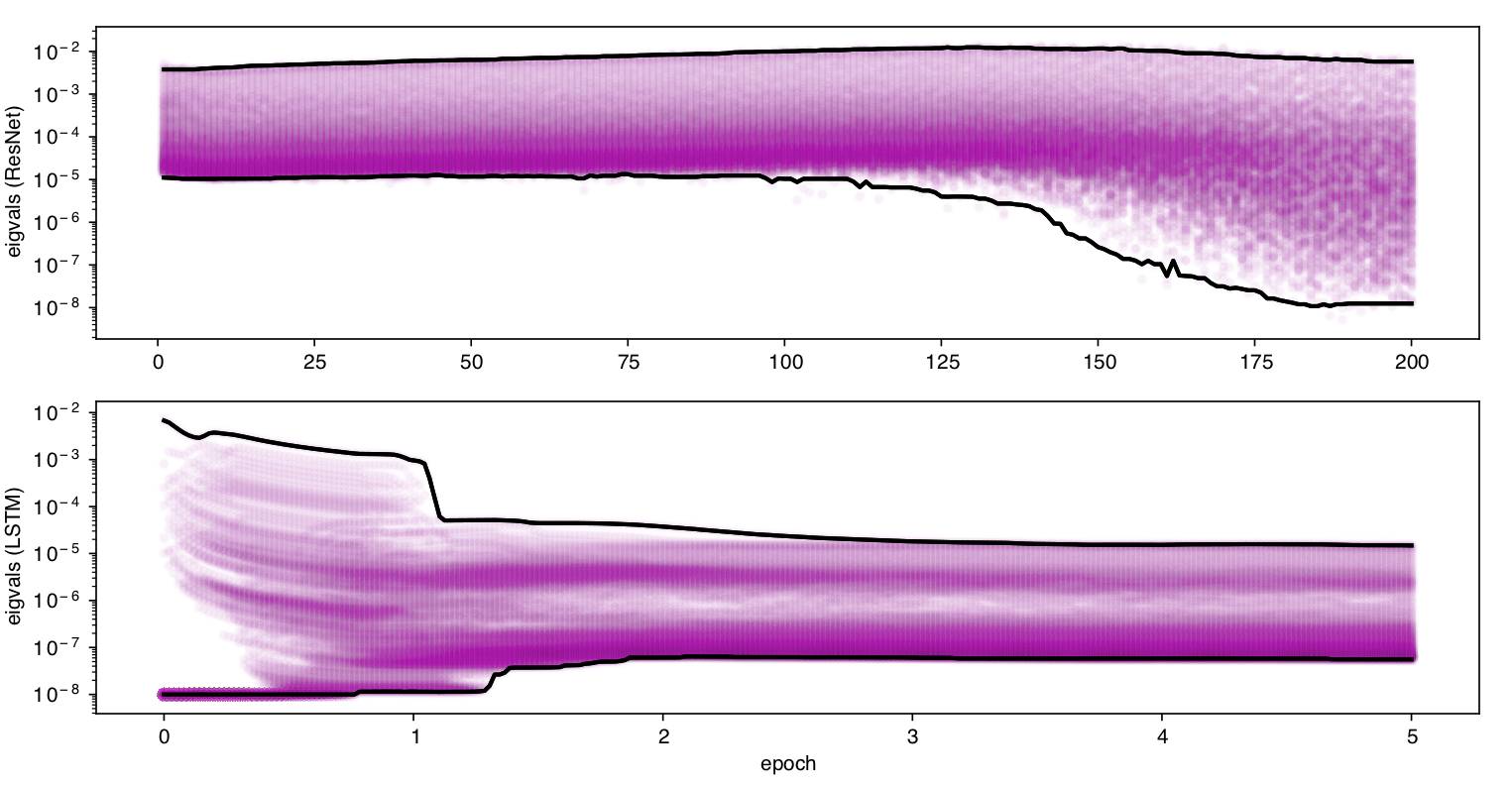}
  \caption{Evolution of the spectrum of the gradient matrix during training.
  Each vertical slice is a density heatmap of the eigenvalues of $\G_t^{\!\top}\!\G_t$.
  The black lines indicate the minimum and maximum eigenvalues, smoothed in time by a median filter.
  \emph{Top:} CNN training. Approaching the end of training, the gradients become more anisotropic. 
  \emph{Bottom:} RNN training. Within the first few epochs, the gradients become more isotropic,
  then stabilize. (Truncated to 5 epochs; the density was visually stable for the remainder of training.) }
  \label{fig:spectra}
\end{figure}


In this section, we unify the insights gleaned from the synthetic experiments and deep learning benchmarks.
Along the way, we provide some interesting anecdotal observations on the evolution of the preconditioner matrices' singular values.

We plot the density of the spectrum of the low-rank preconditioner $\G_t\G_t^{\!\top}$
as training progresses. Since the fast implementation of \ALGNAME{} takes an eigendecomposition of $\G_t^{\!\top}\!\G_t$, we can read off the distribution of eigenvalues during training at no additional computational cost.
Figure~\ref{fig:spectra} visualizes the result of this experiment
for the CNN and RNN training settings from the previous two sections.
In each case, we observe that $\G_t^{\!\top}\!\G_t$ has a condition number of $\sim\! 10^3$,
noting that this can be visualized as the vertical range in the logarithmic plot. 

This visualization affords a new way to see how CNN and RNN landscapes are fundamentally different: their gradient spectra evolve in very distinct ways over the course of training. Interestingly, the condition number of the CNN landscape surges near the end, which may be related to the the low-rank structure of well-trained nets noted by \cite{arora2018stronger}, who derive rank-dependent generalization bounds for neural networks.
On recurrent models, the rapidly evolving spectral structure at the early stage of training indicates a possibly more complex landscape. Intriguingly, the enormous condition number ($\sim\! 10^6$) correlates with the massive lead of \ALGNAME{} over the others, confirming our intuition that full-matrix preconditioning ameliorates anisotropy.

To our knowledge, this is the first empirical study of this kind,
using the covariance matrix of recent gradients as a surrogate to
examining the changing curvature of the loss landscape.
In the spirit of recent empirical lenses of this flavor \citep{raghu2017svcca,li2017visualizing},
we leave this as a way to visualize deep learning dynamics, possibly of independent exploratory interest.


\section{A convergence rate analysis with adaptivity}  \label{sec:theory}

In this section we outline our analysis of \ALGNAME{}, for which we show convergence to an approximate first-order critical point, in some settings faster than SGD. To obtain the strongest theory, we analyze GGT with a ``hard window'' instead of exponentially decaying gradient memory,
explained in Section~\ref{sec:idealization}. 

We work in the usual theoretical framework of stochastic optimization of a differentiable non-convex function $f(\cdot)$,
equipped with an unbiased variance-bounded stochastic gradient oracle
$\widetilde{\nabla} f (\cdot)$.
The objective, as is standard in the literature (see, e.g. \cite{ghadimi2013stochastic,allen2016variance}),
is to find an $\eps$-approximate stationary point $x$; that is, $\|\nabla f(x)\|\leq \eps$.

\subsection{The adaptive ratio}
We quantify the improvement of adaptive regularization by its advantage over the usual worst-case bound of SGD. To this end, we define the \emph{adaptive ratio} $\mu$ of an algorithm $\mathcal{A}$ as 
$$ \mu \defeq \frac{f(x_\mathcal{A}) - f(x^*)}{\|x_1 - x^*\|_2 \cdot \frac{\sigma}{\sqrt{T}}}, $$
where $x_\mathcal{A}$ is the output of the $\mathcal{A}$, and $x^*$ is a comparator. For convex optimization problems $x^*$ is naturally the global minimum. For non-convex optimization it is a subtler choice, which we detail in Appendix \ref{sec:appendix}.

This ratio for the AdaGrad algorithm was shown in \citep{duchi2011adaptive} to be always bounded by a quantity independent of $T$, and potentially much smaller. Specifically, it was shown to be inversely proportional to the dimension in certain convex optimization problems, providing a theoretical justification for the speedup of adaptive optimizers. In Section~\ref{sec:example}, we show a new, simple, and natural setting illustrating adaptive speedup, even for a \emph{strongly convex} function $f$.

\subsection{Adaptive convergence rate guarantee}

We informally state the main theorem below.
We defer the full bound without suppressed smoothness constants, as well as all technical proofs, to Appendix~\ref{sec:appendix}.
\begin{theorem}\label{thm:main-mini}
Let $f : \R^d \rightarrow \R$ be a bounded, Lipschitz, and smooth function with stochastic gradient oracle $\SGDf( \cdot )$, whose variance is at most $\sigma^2$.
In expectation, Algorithm~\ref{alg:nonconvexsmooth} outputs an $\eps$-approximate critical point of $f$, with $\widetilde O\left( \frac{ \mu^2 \sigma^2}{ \eps^4} \right)$ calls to $\SGDf( \cdot )$.
\end{theorem}

This theorem matches and potentially improves the known analysis for stochastic gradient descent with the introduction of the data-dependent adaptivity constant $\mu$ into the leading-order term governing the rate of convergence. Since \cite{duchi2011adaptive} bounded $\mu$ by a quantity independent of $T$, our theorem matches the classic $O\pa{ \eps^{-4} }$ rate of convergence.

\section{Conclusion}
This work investigates full-matrix adaptive regularization: our main contribution is to make this technique viable for large-scale optimization, by a method for efficient multiplication by the inverse square root of a full second-moment matrix over a short window of gradients. This leads to a new algorithm, GGT, a truly scalable optimization algorithm with full-matrix adaptive preconditioning.

Through synthetic experiments, we have shown that GGT accelerates optimization in ill-conditioned loss landscapes; this is supported by accompanying adaptive convergence guarantees. Preliminary experiments show accelerated convergence on standard deep learning benchmarks, with very different training dynamics from existing diagonal adaptive methods. We accompany our algorithm and experiments with the first theoretical characterization of the benefits of adaptive regularization in a non-convex setting. We hope that GGT will be the first of a new class of algorithms for the modern large-scale optimization toolbox, and to foster new discussion towards an ever-elusive understanding of loss landscapes in deep learning.

\section*{Acknowledgments}
We are grateful to Yoram Singer, Tomer Koren, Nadav Cohen, and Sanjeev Arora
for helpful discussions.

\bibliography{main}
\bibliographystyle{alpha}

\appendix
\newcommand{\alglambdadef}{\lambda}
\section{Full adaptive convergence analysis} \label{sec:appendix}
In this section, we give the details on the theoretical treatment of GGT outlined in Section~\ref{sec:theory}. The overall goal is to develop a theory for adaptive regularization in non-convex stochastic optimization. After formalizing the setting, we will define a version of GGT that uses a hard gradient memory window. This will allow us to transfer any insight on the advantage of adaptivity in the convex case to the non-convex case, giving rise to the main theorem. We will conclude this section by with an example illustrating the advantage of adaptive optimizers in the presence of sparse gradients.

\subsection{Setting: stochastic non-convex optimization}
Theorem~\ref{th:th} will provide a bound on the number of stochastic gradient calls required by \ALGNAME{} to achieve a first-order critical point. In particular, the theorem shows that \ALGNAME{} can converge to an approximate first-order critical point faster than SGD,
with convergence rate controlled by the \emph{adaptive ratio} $\mu$, defined in \eqref{eqn:mudef}.

We consider the standard setting of stochastic optimization of a differentiable non-convex function $f(\cdot)$,
equipped with a bounded-variance stochastic gradient oracle defined as follows. 

\begin{definition}[stochastic gradient oracle]
Given a function $f:D \rightarrow \reals$ we call an oracle $O_f$, a $\sigma$-bounded stochastic gradient oracle if for any $x$, $O_f$ returns a a random vector $\SGDf(x)$ such that \[\E\bra{ \SGDf(x) } = \nabla f(x) \quad \text{ and } \quad \E\bra{ \|\SGDf(x) - \nabla f(x)\|^2 } \leq \sigma^2.\]
\end{definition}

The objective, as is standard in non-convex optimization,
is to find a first-order critical point, i.e. a point $x$ for which $\|\nabla f(x)\|\leq \eps$. We will also assume that $f$ has a Lipschitz gradient; i.e. $\|\nabla^2 f(x)\|_2 \leq L$. 

Our algorithm makes a reduction to the case of stochastic convex optimization. The setting formally is that, given a smooth convex function and a $\sigma$-bounded stochastic gradient oracle, the algorithm's aim is to minimize the convex function $f$. Given any algorithm $\mathcal{A}$ we can now define the \emph{adaptive ratio} of the algorithm, referred to as $\mu$, as

\begin{equation}
\label{eqn:mudef}
  \mu \defeq \frac{f(x_\mathcal{A}) - f(x^*)}{\|x_1 - x^*\|_2 \cdot \frac{\sigma}{\sqrt{T}}}
\end{equation}

where $x_\mathcal{A}$ is the output of the algorithm $\mathcal{A}$ and $x^* \in \argmin_x f(x)$, with a total of at most $T$ calls to the stochastic gradient oracle. $\mu$ captures the advantage in convergence rate obtained by the algorithm as compared to the error obtained by vanilla SGD, noting that the denominator is a bound on the error obtained by SGD in the same setting. 

A popular algorithm for stochastic (and in general online) convex optimization is AdaGrad \citep{duchi2011adaptive}. Due to adaptive regularization, AdaGrad can often be advantageous over SGD. We quantify this advantage by the notion of $\mu$ defined above. The bounds of \cite{duchi2011adaptive} imply that $\mu$ can be as small as $\frac{1}{\sqrt{d}}$, depending on the geometry of the optimization problem. An example of this was provided by \cite{duchi2011adaptive} for both the diagonal and the full version of Adagrad. At the end of this section, we provide a different example which shows the same phenomenon even in the case of strongly convex functions.


In the rest of this section we describe Algorithm~\ref{alg:nonconvexsmooth}, which uses AdaGrad (Algorithm~\ref{alg:ag_smooth}) as a subroutine during each window. In this regard, while stating the bounds for our algorithms, we use $\mu$ as an upper bound on the advantage of AdaGrad in each iteration.


\subsection{A suitable abstraction for GGT}
\label{sec:idealization}
As mentioned in Section~\ref{sec:theory}, our analysis uses a slightly idealized version of GGT, which replaces the gradient memory mechanism (governed by $w$ and $\beta_2$) with a \emph{hard} window; i.e., the gradient buffer is \emph{reset} every $w$ steps. This simple modification enables us to develop a more informative theory, in which we benefit directly from the familiar theory of AdaGrad for convex optimization, while capturing the necessity of forgetting past gradient information in adaptive non-convex optimization.

First, for clarity, we restate the definition of the full-matrix AdaGrad algorithm, introduced by \cite{duchi2011adaptive}, which accumulates the second-moment matrix of all past gradients:

\begin{algorithm}
\caption{AdaGrad for convex optimization \citep{duchi2011adaptive}}
\begin{algorithmic}[1]
\INPUT initializer $x_1$, window length $w$, stochastic gradient oracle $\SGDf(\cdot)$, $\eps, \eta > 0$.
\For{$t = 1, \ldots, w$}
  \State Receive stochastic gradient $\SGDf(x_t)$.
  \State Let $\G_t = \bra{ g_t \; g_{t-1} \ldots g_{1} }$, where $g_t := \SGDf(x_t)$.
  \State Update $x_{t+1} \leftarrow x_t - \eta \cdot \bra{ \eps \bI + (\G_t\G_t^\top)^{1/2} }^{-1} g_t.$
\EndFor
\OUTPUT Average iterate $\frac{1}{w}\left(\sum_{t=1}^{w} x_t\right)$.
\end{algorithmic}
\label{alg:ag_smooth}
\end{algorithm}

The final algorithm we analyze simply runs AdaGrad between restarts.

\begin{algorithm}
\caption{GGT with a hard gradient window}
\begin{algorithmic}[1]
\INPUT initializer $x_1$, time horizon $T$, window length $w$, $\lambda > 0$.
\For{$t = 1$ to $T$: }
  \State Let $f_t(x) = f(x) + \alglambdadef \|x - x_t\|^2$.
  \State Update $x_{t+1}$ to be the output of Algorithm~\ref{alg:ag_smooth} on $f_t(x)$, starting at $x_t$, for $w$ steps.
\EndFor
\OUTPUT Best iterate $x_{t^*}$, where $t^* := \argmin_{t\leq T+1} \|\nabla f(x_t)\|$.
\end{algorithmic}
\label{alg:nonconvexsmooth}
\end{algorithm}

The remaining discrepancies between Algorithm~\ref{alg:nonconvexsmooth} and Algorithm~\ref{alg:ours} from the main paper are standard. We provide some references below.
\begin{itemize}
\item \textbf{Absence of first-moment estimation.} Although it is customary to use nonzero $\beta_1$ (otherwise known as momentum) when applying Adam in practice, it is orthogonal to the effect of adaptive regularization in all established theory. In fact, the convergence rates given by \cite{kingma2014adam} (and fixed by \cite{reddi2018convergence}) contain only factors of $1/(1 - \beta_1)$, and are thus strongest when $\beta_1 = 0$.
\item \textbf{Model averaging.} Theoretical guarantees in online and stochastic convex optimization are most naturally stated on the average iterate; see \citep{polyak1992acceleration,duchi2011adaptive}. Thus, we adopt the convention that Algorithm~\ref{alg:ag_smooth} returns the average iterate. We note that model averaging is a common regularization technique in practical non-convex settings, though not the default choice for adaptive optimizers in practice.
\item \textbf{$\ell_2$ regularization.} The addition of the $\lambda \norm{x - x_t}^2$ term in Algorithm~\ref{alg:nonconvexsmooth} is an artifact we introduce to obtain a tight analysis for hard-window GGT. It ensures that iterates in each window do not move too far, and allows us to analyze each window as a fixed convex program, so that we can use the convex theory of AdaGrad directly. The soft-window analogue would simply to be decrease the learning rate. Interestingly, a similar technique directly appears in the algorithm proposed by \cite{allen2017katyusha}. Finally, we note that from a $\sigma$-bounded stochastic gradient oracle for $f$, it is trivial to construct one for $f_t$, by adding $-2\lambda x_t$ (deterministically).
\end{itemize}

\subsection{Main theorem and proof}



\begin{theorem}\label{th:th}
Consider a non-convex function $f$, such that for all $x,\; \|\nabla^2 f(x)\|_2 \leq L$ and a point $x_1$ such that $f(x_1) - \min_{x^*\in\mathcal{K}}f(x^*) \leq M$. Further, suppose we have access to a $\sigma$-bounded stochastic gradient oracle $O_f$. Suppose for any $\lambda \geq \frac{L}{2}$, Algorithm~\ref{alg:nonconvexsmooth} is run with $T=\frac{4M(L + 2\lambda)}{\epsilon^2}$ and $w = \frac{16\mu^2 \sigma^2(L + 2\lambda)}{\epsilon^2(2\lambda - L)}$. Then the point $x'$ returned by Algorithm~\ref{alg:nonconvexsmooth} is such that
  \[ \mathbb{E} \|\nabla f(x')\| \leq \eps,\]
  where $\mu = \max_{t \in [T]} \mu_t$ and $\mu_t$ is the adaptive ratio when run on $f_t$ (as defined in \eqref{eqn:mudef}). Further, note that choosing $\lambda = 3L/2$, the total number of stochastic gradient calls to the oracle $O_f$, made by the algorithm is bounded by $T \cdot w = \frac{ 512 LM \mu^2 \sigma^2 }{ \eps^4 }$.
\end{theorem}

For the setting of Theorem \ref{th:th}, the best known bound on the number of oracle calls to the stochastic gradient oracle in the case of the vanilla SGD algorithm is
$O(\frac{LM\sigma^2}{\epsilon^4})$. Note that due to the presence of $\mu^2$ in the bound provided in Theorem \ref{th:th} reflects the advantage of Algorithm \ref{alg:nonconvexsmooth} over SGD. This advantage as we argue in the following section can be as large as up to a factor of $1/d$, a significant improvement over SGD. 

Before proving Theorem \ref{th:th}, we state an oracle complexity bound for AdaGrad (Algorithm \ref{alg:ag_smooth}) for strongly convex functions.

\begin{lemma}\label{thm:sc_ag}
Suppose f is a $\lambda$-strongly convex function equipped with a $\sigma$-bounded stochastic gradient oracle. Given an initial point $x_1$, Algorithm~\ref{alg:ag_smooth} when run for $w$ steps is guaranteed to output a point $x'$ such that
\[\E[f(x')] - \min_{x}f(x) \leq \frac{\mu^2 \sigma^2 \sqrt{2(f(x_1) - \min_x f(x) )}}{\sqrt{\lambda w}},\]
where $\mu$ is the adaptive ratio of AdaGrad on $f$ as defined in \eqref{eqn:mudef}.
\end{lemma}

Using this lemma we first prove Theorem \ref{th:th} and then finish the section by providing a proof of Lemma \ref{thm:sc_ag}.

\begin{proof}[Proof of Theorem \ref{th:th}]
We begin by proving the following useful property regarding the function $f_t$ for any $t$ and any $\eta$:
 \begin{align*}
    f_t(x_t) - \min_x f_t(x) &\geq f(x_t) - f_t(x_t - \eta \nabla f(x_t)) \\
    &= f(x_t) - f(x_t - \eta \nabla f(x_t)) - \alglambdadef \eta^2 \|\nabla f(x_t)\|^2 \\
    &\geq \eta \|\nabla f(x_t)\|^2 - \frac{L\eta^2}{2} \|\nabla f(x_t)\|^2 - \alglambdadef \eta^2 \|\nabla f(x_t)\|^2.
  \end{align*}
  Setting $\eta = \frac{1}{L + 2\alglambdadef}$, we get that
  \begin{equation}
  \label{eqn:intermed}
   f_t(x_t) - \min_x f_t(x) \geq \frac{\|\nabla f(x_t)\|^2}{2(L + 2\alglambdadef)}.     
  \end{equation}
  We will now prove the theorem by contradiction. Suppose for all the $t$, $\|\nabla f(x_t)\|^2 > \epsilon^2$. We now have that 
  \begin{align}
    f(x_t) - f(x_{t+1}) &\geq f_t(x_t) - f_t(x_{t+1}) \nonumber\\
    &= f_t(x_t) - \min_x f_t(x) - (f_t(x_{t+1}) - \min_x f_t(x)) \nonumber\\
    &\geq f_t(x_t) - \min_x f_t(x) - \frac{\sqrt{f_t(x_t) - \min_x f_t(x)} \sqrt{\epsilon^2}}{2\sqrt{2(L + 2\lambda)}} \nonumber\\
    &\geq f_t(x_t) - \min_x f_t(x) - \frac{\sqrt{f_t(x_t) - \min_x f_t(x)} \sqrt{\|\nabla f(x_t)\|^2}}{2\sqrt{2(L + 2\lambda)}} \nonumber\\
    &\geq \frac{f_t(x_t) - \min_x f_t(x)}{2} \geq \frac{\|\nabla f_t(x_t)\|^2}{4(L + 2\lambda)} > \frac{\epsilon^2}{4(L + 2\lambda)}. \label{eqn:derive}
  \end{align}
where the first inequality follows from noting that $f(x) \leq f_t(x)$ for all $x \in \R^d$, and that $f_t(x_t) = f(x_t)$. The second inequality follows from Lemma \ref{thm:sc_ag}, by noting that $f_t$ is $2\lambda - L$ strongly convex and the choice of $w$. The third inequality follows from the counterfactual assumption $\|\nabla f(x_t)\|^2 > \epsilon^2$, and the last set of inequalities follow from \eqref{eqn:intermed}.

Summing \eqref{eqn:derive} over all $t \in [T]$ gives us that
\[ f(x_1) - f(x_{T+1}) > \frac{T \epsilon^2}{4(L + 2\lambda)} = M,\]
which is a contradiction and hence proves the theorem. The number of stochastic gradient oracle calls when $\lambda = 3L/2$ is bounded by 
\[ T \cdot w \leq \frac{4M(L + 2\lambda)}{\epsilon^2} \cdot \frac{16\mu^2 \sigma^2(L + 2\lambda)}{\epsilon^2(2\lambda - L)} \leq \frac{512 M \mu^2 \sigma^2 L}{\epsilon^4}.\]

\end{proof}

\begin{proof}[Proof of Lemma \ref{thm:sc_ag}]
We have
\begin{align*}
  \E\bra{ f(x') } - f(x^*) &= \E \bra{ \frac{\mu \sigma}{\sqrt{w}} \|x_1 - x^*\| } \leq \E\bra{ \frac{\mu \sigma\sqrt{2(f(x_1)- f(x^*))}}{\sqrt{w \lambda}} },\\
\end{align*}
where the equality follows from the definition of $\mu$ and the inequality follows from strong convexity. 
\end{proof}

\subsection{Example: the advantage of adaptivity}
\label{sec:example}


Here we provide a strongly convex function (in fact a simple quadratic) and a sketch of the proof of the fact that depending on the starting point adaptive advantage i.e. $\mu$ of AdaGrad can be up to a factor of $\sqrt{d}$.

Consider the function $\|x\|^2$ in $\reals^d$ and consider the starting point $x_0$. Let the stochastic gradient oracle $O_f$ be such that before the experiment the oracle samples a random orthonormal basis $V = \{v_1 \ldots v_d\}$ and when queried at a point $x$ returns the vector 
\[ \SGDf(x) = \nabla f(x) + a_t z_t\]
where $a_t = \pm 1$ with probability $1/2$ and $z_t$ is a vector picked from the set $V$ uniformly randomly. It is easy to verify that $O_f$ is a $\sigma$-bounded stochastic gradient oracle. We now provide an analysis of AdaGrad with the above oracle for $f$. 

Firstly note that we can without loss of generality, assume that the basis chosen is the canonical basis $\{e_i\}$. This can be seen by performing a simple rotation which does not affect the function $\|x\|^2$. Further under this setting note that AdaGrad is equivalent to running a one dimensional SGD algorithm in each coordinate independently. The following bound now follows directly from the well known analysis of SGD on smooth functions (see Theorem 6.3 in \cite{bubeck2015convex} for a concrete reference).
\[ \forall\;\;i \in [d]\;\;\;\;(x'[i])^2 \lesssim \frac{|x_1[i]|\sqrt{\sum_{t = 1}^T (\sigma_t[i])^2 }}{T} = |x_1[i]| \cdot \sqrt{\frac{1}{Td} },\]
where $\sigma_t[i] = 1/d$ is the variance of the noise in the stochastic gradient seen at time $t$ along coordinate $i$ and $x'$ is the output of AdaGrad. Note that in the above we have ignored the \emph{bias} term which scales as $1/T$ (refer to Theorem 6.3 in \cite{bubeck2015convex}). This implies that the overall error for AdaGrad scales as
\[\|x'\|^2 \lesssim \|x_1\|_1 \cdot \sqrt{\frac{1}{Td}}.\]
Therefore the advantage of adaptivity $\mu$ is bounded by
\[ \mu \leq \frac{\|x_1\|_1 \sqrt{\frac{1}{Td}}}{\|x_1\|_2 \sqrt{\frac{1}{T}}} = \frac{\|x_1\|_1}{\|x_1\|_2\sqrt{d}}.\]
This follows by noting that the variance of the noise in the stochastic gradient measured in the $\ell_2$ is $1$. The above expression implies that $\mu$ can be as small as $O(\frac{1}{\sqrt{d}})$ in particular if the starting point $x_1$ is sparse or nearly sparse and therefore $\|x_1\|_1 \sim \|x_1\|_2$.


\section{Additional experiments}
\label{sec:appendix-experiments}

\subsection{Experiments on additional architectures}
\begin{figure}
  \centering
  \includegraphics[width=0.95\textwidth]{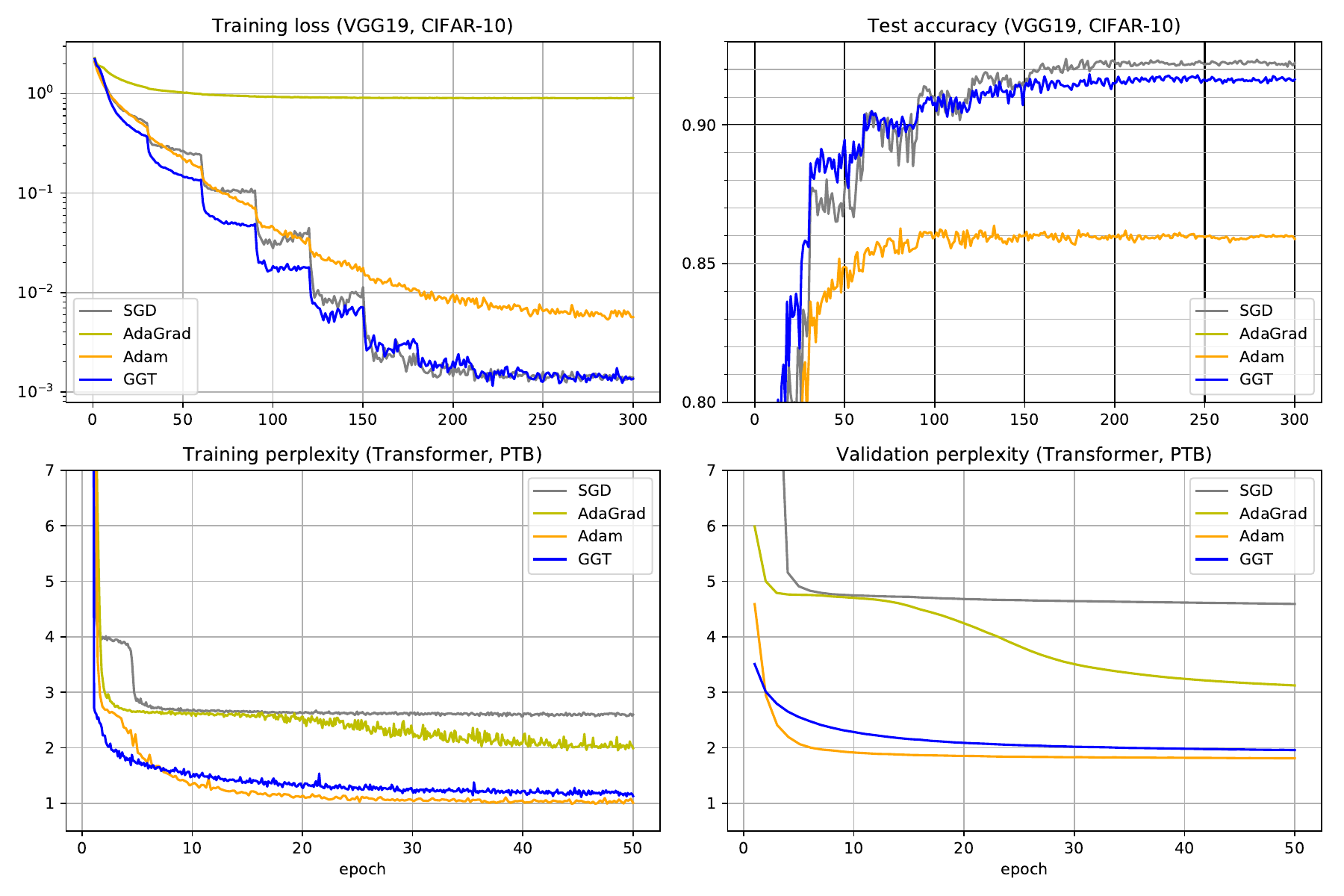}
  \caption{Plot of experiments from Sections~\ref{subsection:experiments-vision} and \ref{subsection:experiments-language}, with wall clock time on horizontal axis instead of epoch count (as in Figure~\ref{fig:cnn-rnn}).
  \emph{Top:} CIFAR-10 classification with a 19-layer vanilla CNN.
  \emph{Bottom:} PTB character-level language modeling a Transformer network.
  }
  \label{fig:more-experiments}
\end{figure}

We present some additional large-scale empirical studies in Figure~\ref{fig:more-experiments}.

To demonstrate a vision task with a harder optimization landscape, we use GGT to train a 19-layer ``vanilla'' convolutional network (VGGNet, \cite{simonyan2014very}), without residual connections or batch normalization, on the same CIFAR-10 classification task. Here, we recover the same insights as found by \cite{wilson2017marginal}, in which diagonal-matrix adaptive methods can fail to train a network dramatically. Here, unlike diagonal-matrix adaptive optimizers, GGT stays on par with SGD throughout training, with a $\sim 1\%$ gap remaining in generalization at the end. We use a standard fixed halving learning rate schedule; it is clear here that in the initial epochs after decaying the learning rate, GGT trains the most rapidly. We leave a careful investigation of leveraging this phenomenon, and tuning GGT's learning rate schedule, to future work.

A recent significant advancement on many NLP tasks, including language modeling, is the introduction of attention-based models. We investigate the behavior of GGT on a Transformer network \citep{vaswani2017attention}, on the same Penn Treebank character-level language modeling task. Here, after an initial lead, GGT is outperformed by Adam in training and validation loss. The value of using gradient correlations to assist in the training of attention models seems to be limited.

\end{document}